\documentclass{article}

\usepackage[utf8]{inputenc} % allow utf-8 input
\usepackage[T1]{fontenc}    % use 8-bit T1 fonts
\usepackage{hyperref}       % hyperlinks
\usepackage{url}            % simple URL typesetting
\usepackage{booktabs}       % professional-quality tables
\usepackage{amsfonts}       % blackboard math symbols
\usepackage{nicefrac}       % compact symbols for 1/2, etc.
\usepackage{microtype}      % microtypography

\usepackage{mathtools}
\usepackage{amsmath}
\usepackage{amssymb}
\usepackage{amsthm}

\usepackage{enumerate}

\DeclareMathOperator*{\argmax}{arg\,max}
\DeclareMathOperator*{\argmin}{arg\,min}

\usepackage{algorithm}
\usepackage[noend]{algpseudocode}

\usepackage{graphicx}

\newtheorem{lemma}{Lemma}[section]
\newtheorem{prop}[lemma]{Proposition}

\newtheorem{theorem}[lemma]{Theorem}
\newtheorem{cor}[lemma]{Corollary}

\theoremstyle{definition}

\theoremstyle{remark}

\newtheorem{remark}[lemma]{Remark}

\mathtoolsset{showonlyrefs}

\newcommand{\R}{\mathbb{R}}

\title{Gaussian Mixture Model Decomposition of Multivariate Signals}

\author{
	Gustav Zickert\thanks{Corresponding author.}\\
	Department of Mathematics\\
	KTH Royal Institute of Technology\\
	\texttt{zickert@kth.se} \vspace{10pt} \\ 
	Can Evren Yarman \\
	Schlumberger Cambridge Research \\
	\texttt{cyarman@slb.com} \\
}

\begin{document}
	
	\maketitle
	
	\begin{abstract}
		We propose a greedy variational method for decomposing a non-negative multivariate signal as a weighted sum of Gaussians, which, borrowing the terminology from statistics, we refer to as a Gaussian mixture model. Notably, our method has the following features: (1) It accepts multivariate signals, i.e. sampled multivariate functions, histograms, time series, images, etc. as input. (2) The method can handle general (i.e. ellipsoidal) Gaussians. (3) No prior assumption on the number of mixture components is needed. To the best of our knowledge, no previous method for Gaussian mixture model decomposition simultaneously enjoys all these features. We also prove an upper bound, which cannot be improved by a global constant, for the distance from any mode of a Gaussian mixture model to the set of corresponding means. For mixtures of spherical Gaussians with common variance $\sigma^2$, the bound takes the simple form $\sqrt{n}\sigma$. We evaluate our method on one- and two-dimensional signals. Finally, we discuss the relation between clustering and signal decomposition, and compare our method to the baseline expectation maximization algorithm.
	\end{abstract}

\section{Introduction}
	Mixtures of Gaussians are often used in clustering to fit a probability distribution to some given sample points. In this work we are concerned with the related problem of approximating a non-negative but otherwise arbitrary signal by a sparse linear combination of potentially anisotropic Gaussians. Our interest in this problem stems mainly from its applications in transmission electron microscopy (TEM). The output of the TEM reconstruction procedure is a 3D voxelized structure that ideally represents the electrostatic potential of the imaged specimen. Via a process known in the TEM-community as \emph{coarse-graining}, it is common to express this 3D structure as a linear combination of Gaussians \cite{Kaw18, Jou16, JS16a, JS16b}. This speeds up post-processing tasks such as fitting an atomic model to the structure \cite{Kaw18}, but one can also use coarse-graining as a form of denoising \cite{JS16b}.
	
	Methods for sparse decomposition of multivariate signals as GMMs may be considered in two classes. The first class contains methods based on the expectation maximization algorithm that is commonly used for fitting a GMM to a point-cloud, and adapt it to input data in the form of multivariate signals \cite{Kaw18, Jou16}. The second class is the class of greedy variational methods \cite{FC99, JS16a, KBGP18}. The proposed method, which belongs to the latter class, is similar to \--- and is inspired by \--- both \cite{KBGP18} and \cite{JS16a}. It is a continuously parameterized analogue of orthogonal matching pursuit where at each iteration the $L^2$-norm of the error is non-increasing. The significance of the proposed method, and what distinguishes it from earlier work on GMM decomposition of signals, is that
	\begin{enumerate}
		\item The resulting GMM may contain ellipsoidal Gaussians. This allows for a sparser representation than what could be achieved with a GMM only consisting of spherical Gaussians. 
		\item The number of Gaussians does not need to be set before-hand.
	\end{enumerate}
	We are not aware of previous methods for GMM decomposition of signals that enjoys both of these properties at the same time.  
	 
	We complement our algorithm with a theorem (Theorem \ref{thm:modeDist}) that upper-bounds the distance from a local maximum of a GMM to the set of mean vectors. This provides theoretical support for our initialization of each new mean vector at a maximum of the residual. We remark that Theorem \ref{thm:modeDist} could also be of interest in its own right; whereas the number of modes of Gaussian mixtures have been investigated previously \cite{CP00, amendola2017maximum}, the authors of this paper are not aware of any existing quantitative bounds on the distance from a mode of a GMM to its mean vectors in the multivariate setting.
	
	The rest of this paper is organized as follows. In section \ref{sec:probStatement} we define the GMM decomposition problem. Section \ref{sec:method} contains a description of the proposed algorithm together with numerical examples. Section \ref{sec:theory} is devoted to theoretical questions, in particular we state and prove the afore-mentioned upper bound. Finally, in section \ref{sec:conc} we provide a conclusion.
	
	\section{Problem statement}\label{sec:probStatement}
		For $x_0\in \R^n$ and $\Sigma\in \R^{n\times n}$ symmetric non-negative definite we define $g(x_0, \Sigma)$ as the Gaussian density in $n$ dimensions with mean vector $x_0$ and covariance matrix $\Sigma^{2}$, i.e.
		\begin{align}
		g(x_0,\Sigma)(x) := C_{\Sigma}\exp\left\lbrace-\frac{1}{2}\left(x-x_0\right)^T\Sigma^{-2}\left(x-x_0\right)\right\rbrace,
		\end{align}
		where $C_{\Sigma}$ is a normalizing factor, ensuring that $\left\lvert g(x_0,\Sigma)\right\rvert_1 =1$.
		Further, by a Gaussian mixture model (GMM) we mean a linear combination of the form\footnote{We do not require the GMM to be normalized, i.e. we do not require that $\sum_{m=1}^Ma_m=1$.}
		\begin{align}
		\sum_{m=1}^Ma_mg(x_m,\Sigma_m), \quad a_m>0, M>0.
		\end{align}
		
		The problem we consider is to construct an algorithm with the following properties. Given input in the form of a non-negative signal $d\in\R^{k_1\times\dots\times k_n}$, where $k_i$ is the number of grid-points along the $i$:th variable, the output should be a list of GMM parameters, i.e. it is a list $(a^*_m, x_m^*, \Sigma^*_m)_{m=1}^M$ of weights, mean vectors and square roots of covariance matrices. The output should be such that 
		\begin{enumerate}
			\item the residual \begin{align}
			r:=d- \sum_{m=1}^Ma^*_mg(x^*_m,\Sigma^*_m) 
			\end{align}
			has a small $L^2$-norm. (We remark that any sufficiently regular non-negative function can be uniformly approximated arbitrarily well using GMMs; see section \ref{sec:theory}.) 
			\item the approximation is sparse, i.e. the number of Gaussians $M$ in the sum should ideally be as small as possible given the $L^2$-norm of the residual.
		\end{enumerate}

	\section{Proposed method}\label{sec:method}
	In each iteration of our algorithm, a new Gaussian is added to the GMM by a procedure that corresponds to one iteration of a continuously parametrized version of matching pursuit (MP), c.f. \cite{mallat1993matching}: The starting guess $x_0$ for the mean vector is defined to be a global maximum of $r'$, which is a smoothed version of the current residual $r$.  Likewise $\Sigma_0$ is set to be a square root of the matrix of second-order moments of $\frac{r''}{|r''|_1}$, where $r''$ is $r$ restricted to a neighborhood of $x_0$. The initial weight $a_0$ is given by projecting the Gaussian atom defined by $(x_0, \Sigma_0)$ onto the residual, i.e. it is given as the global minimum of the convex objective $a\to \lvert r-ag(x_0, \Sigma_0)\rvert_2^2, a\geq 0$.
	
	We found that this vanilla MP-type of approach was by itself not capable of producing a good approximation. Because of this, in our method the already obtained Gaussians are updated  in a way that bears some resemblance to the projection step of orthogonal MP (OMP) \cite{pati1993orthogonal}: Starting from $(a_0, x_0, \Sigma_0)$, the parameters defining the most recently added Gaussian are updated by minimizing the $L^2$-norm of the residual.\footnote{This step is not crucial for the performance of the algorithm, but was empirically found to improve the final data-fit.} The final step of an outer iteration is to simultaneously adjust all Gaussians in the current GMM, again by minimizing the $L^2$-norm of the residual. We propose to use the L-BFGS-B\cite{byrd1995limited} method with non-negativity constraints on the weights for all of the three afore-mentioned minimization problems. Our algorithm runs until some user-specified stopping criterion is met and is summarized in pseudo-code in Algorithm~\ref{alg}.
	
	We now assess the space- and time complexity of the proposed method in terms of the final number $M$ of Gaussians used and the dimension $n$ of the input signal. Let $N_{i}$ denote the number of parameters of $i$ Gaussians, so $N_i = \mathcal{O}(in^2)$. The most demanding step of a single outer iteration of the proposed algorithm is the simultaneous adjustment of all Gaussians. If this step is done using L-BFGS-B, this amounts to a cost of $\mathcal{O}(N_i)$, both in terms of space and time \cite{byrd1995limited}. Summing over the $M$ iterations we hence find the space-- and time requirements of the proposed method to be $\mathcal{O}(M^2n^2)$. 

	\begin{algorithm}
		\caption{}\label{alg}
		\begin{algorithmic}[1]
			\Procedure{GMM decomposition}{}\newline
			\textbf{Input:} Signal $d$, hyper-parameters $\tau_1, \tau_2\in \mathbb{Z}_+$\newline%, accuracy $\epsilon>0$\newline 
			\textbf{Output:} GMM parameters $\left(a^*_m, x^*_m, \Sigma^*_m\right)_{m=1}^M$
			\State{$M \gets 0$, $r\gets d$}
			\While{stopping criterion is not met}
			\State{$M\gets M+1$}
			\State{$ r'(y) \gets \frac{1}{\tau_1} \sum_{z\in A(y)}r(z)$,\par  \hskip\algorithmicindent\quad $A(y) := \{\tau_1$ grid points closest to $y\}$}
			\State{$x_0 \gets \argmax_{y} r'(y)$}
			\State{$r'' \gets $ restriction of $r$ to the $\tau_2$ grid points \par  \hskip\algorithmicindent\quad closest to $x_0$}
			%		\State{$m_0 \gets \argmax_x r_\text{smooth}(x)$}
			\State{$\Sigma_0 \gets $ square root of matrix of second order \par  \hskip\algorithmicindent\quad moments of $\frac{r''}{|r''|_1}$}
			\State{$a_0 \gets \argmin_{a\geq0} \lvert r-ag(x_0, \Sigma_0)\rvert_2^2$}
		\State{$\left(a^*, x^*, \Sigma^*\right) \gets \underbrace{\text{loc. min.} \left\lvert r-ag(x, \Sigma)\right\rvert_2^2, a\geq0}_{\text{initialized at }a_0, x_0, \Sigma_0}$}
			\State{$\left(a^*_M,x^*_M,\Sigma^*_M\right) \gets \left(a^*,x^*,\Sigma^*\right)$}
			\State{$\left(a^{**}_m, x^{**}_m, \Sigma^{**}_m\right)_{m=1}^M \gets \left(a^*_m, x^*_m, \Sigma^*_m\right)_{m=1}^M$}
			\State{$\left(a^*_m, x^*_m, \Sigma^*_m\right)_{m=1}^M \gets$ \par  \hskip\algorithmicindent\quad $\underbrace{\text{loc. min.} \left\lvert d-\sum_{m=1}^M a_mg(x_m, \Sigma_m)\right\rvert_2^2, a_m\geq0}_{\text{initialized at } \left(a^{**}_m, x^{**}_m, \Sigma^{**}_m\right)_{m=1}^M}$ }
			\State{$r \gets d-\sum_{m=1}^M a^*_mg(x^*_m, \Sigma^*_m)$}
			%		\State{$r_\text{smooth} \gets k*r$}
			\EndWhile
			\EndProcedure
		\end{algorithmic}
	\end{algorithm}
	
	\subsection{Main numerical results}\label{sec:mainNumerical}
	\newcommand{\ra}[1]{\renewcommand{\arraystretch}{#1}}	
	
		\begin{table}\centering
			\ra{1.3}
			\caption{Input and output GMM parameters for Exp. 1} \label{tab:Exp1}
			\begin{tabular}{@{}lllclll@{}}\toprule
				\multicolumn{3}{c}{Input} & \phantom{a}&  \multicolumn{3}{c}{Output} \\
				\midrule 
				$a_m$ & $x_m$ & $\Sigma^2_m$ && $a_m$ & $x_m$ & $\Sigma^2_m$    \\
				\hline 
				$1$ & $0$  & $1$ && $5.0272$ & $1.3052$  & $2.2632$ \\ 
				$8$ & $0$  & $4$ && $1.0116$ & $-7.9756$ & $0.9512$ \\
				$1$ & $-2$ & $1$ && $1.0351$ & $8.0478$  & $1.1041$ \\
				$1$ & $2$  & $1$ && $5.8708$ & $-1.1215$ & $2.4119$ \\
				$1$ & $-8$ & $1$ && & & \\
				$1$ & $8$  & $1$ && & &  \\
				\bottomrule
			\end{tabular}
		\end{table}

		\begin{table}\centering
			\ra{1.3}
			\caption{Input and output GMM parameters for Exp. 2} \label{tab:Exp2}
			\begin{tabular}{@{}lll@{}}\toprule
				\multicolumn{3}{c}{Input} \\% & \phantom{a}&   
				\midrule 
				$a_m$ & $x_m$ & $\Sigma^2_m$    \\
				\hline 
				$2$ & $(-1.5, -2.5981)$ & $\begin{psmallmatrix}
				0.7969 & 1.272 \\
				1.272 & 2.2656
				\end{psmallmatrix}$  \\ 
				$2$ & $(-1.5, 2.5981)$ & $\begin{psmallmatrix}
				0.7969 & -1.272 \\
				-1.272 & 2.2656
				\end{psmallmatrix}$  \\
				$2$ & $(3, 0)$ & $\begin{psmallmatrix}
				3 & 0 \\
				0 & 0.0625
				\end{psmallmatrix}$  \\
				$1$ & $(-1.75, -3.0311)$ &$\begin{psmallmatrix}
				1 & 0 \\
				0 & 1
				\end{psmallmatrix}$  \\
				\toprule
				\multicolumn{3}{c}{Output} \\
			    \midrule
				$a_m$ & $x_m$ & $\Sigma^2_m$ \\ 
				\hline
				 $1.9833$ & $(-1.4937, -2.5862)$ & $\begin{psmallmatrix}
				0.5758 & 1.1181 \\
				1.1181 & 2.4908
				\end{psmallmatrix}$ \\
				
				$1.9901$ & $(3.0075, -9.7258\times 10^{-4})$ & $\begin{psmallmatrix}
				2.982 & 0.0082\\
				0.0082 & 0.0623  
				\end{psmallmatrix}$ \\
				
				$2.001$ & $(-1.4994,  2.5924)$ & $\begin{psmallmatrix}
				0.7858 &-1.2663 \\
				-1.2663  &2.2805 
				\end{psmallmatrix}$  \\
				
				$1.0101$ & $(-1.7466, -3.038)$ &$\begin{psmallmatrix}
				0.9878& 0.0126 \\
				0.0126 &0.9731 
				\end{psmallmatrix}$  \\
				\bottomrule
			\end{tabular}
		\end{table}

		\begin{table}\centering
			\ra{1.3}
			\caption{Input and output GMM parameters for Exp. 3} \label{tab:Exp3}
			\begin{tabular}{@{}lll@{}}\toprule
				\multicolumn{3}{c}{Input} \\% & \phantom{a}&   
				\midrule 
				$a_m$ & $x_m$ & $\Sigma^2_m$  \\
				\hline 
				$5$ & $(-5, 5)$  & $\text{Id}_2$ \\ 
			    $1$ & $(5, -5)$  & $\text{Id}_2$ \\
				$3$ & $(5, 5)$   & $\text{Id}_2$ \\
				$4$ & $(-5, -5)$ & $\text{Id}_2$ \\
				$5$ & $(-2, 0)$  & $\text{Id}_2$ \\
				$5$ & $(0, -2)$  & $\text{Id}_2$ \\
				$5$ & $(2, 0)$   & $\text{Id}_2$ \\
				$5$ & $(0, 2)$   & $\text{Id}_2$ \\
				\toprule
				\multicolumn{3}{c}{Output} \\
				\midrule
				$a_m$ & $x_m$ & $\Sigma^2_m$  \\
				\hline 
				$4.9685$ & $(9.7331\times 10^{-4}, 2.0079)$ & $\begin{psmallmatrix}
				0.9781 &0.0125 \\
				0.0125 &1.012  
				\end{psmallmatrix}$ \\
				
				$4.9517$ & $(-4.9988,  5.0024)$ & $\begin{psmallmatrix}
				1.0006 &0.0127 \\
				0.0127 &1.0103 
				\end{psmallmatrix}$ \\
				
				$3.9821$ & $(-5.005 , -4.9978)$ & $\begin{psmallmatrix}
				0.9919 &0.0017 \\
				0.0017 &0.9731 
				\end{psmallmatrix}$ \\
				
				$2.9853$ & $(4.9935, 5.0071)$ & $\begin{psmallmatrix}
				0.9955 &-0.0308 \\
				-0.0308  &0.9976 
				\end{psmallmatrix}$ \\
				
				$5.0608$ & $(-1.9916,  0.0076)$ & $\begin{psmallmatrix}
				1.0128 &0.0154 \\
				0.0154 &0.9977 
				\end{psmallmatrix}$ \\
				
				$5.0322$ & $(0.0034, -1.9952)$ & $\begin{psmallmatrix}
				0.9968& 0.0145 \\
				0.0145 &1.0038 
				\end{psmallmatrix}$ \\
				
				$1.0048$ & $(5.0158, -5.0174)$ & $\begin{psmallmatrix}
				1.007 & 0.0083 \\
				0.0083 &1.0034 
				\end{psmallmatrix}$ \\
				
				$4.9247$ & $(2.0213, 0.0077)$ & $\begin{psmallmatrix}
				0.9743& -0.0097 \\
				-0.0097 & 0.9997  
				\end{psmallmatrix}$ \\
				
				\bottomrule
			\end{tabular}
		\end{table}

	As proof of concept for the proposed method, we ran our algorithm on small toy-examples in 1D and 2D. We refer to the examples as Experiment 1 \--- Experiment 3, see Figures  \ref{fig:1d_exa}\---\ref{fig:2d_exa_8_circs}. A clean signal $d_\text{clean}$ was generated by discretizing GMMs with parameters as in Table \ref{tab:Exp1} \--- \ref{tab:Exp3} to the grids $\left\{y_k = -10 + \frac{20k}{1000}\right\}_{k=0}^{1000}$ and $\left\{y_{k,\ell} = \left(-10 + \frac{20k}{65}, -10 + \frac{20\ell}{65}\right)\right\}_{k,\ell=0}^{64}$ in the one- and two-dimensional experiments respectively. From this signal, noisy data was then generated as $d = d_\text{clean} + \epsilon$ where $\epsilon$ denotes white Gaussian noise with standard deviation $\sigma_\text{noise}$. The latter was chosen so that SNR = 20, where we define SNR by:
	\begin{align}
	\text{SNR} &:= 10\log_{10}\frac{\text{Var}\left(d_\text{clean}\right)}{\sigma^2_\text{noise}}.
	\end{align}
	As stopping criterion we used $\text{SNR}_\text{stop} \geq 20$, where 
	\begin{align}
	\text{SNR}_\text{stop} &:= 10\log_{10}\frac{\text{Var}\left(d_\text{est}\right)}{\text{Var}\left(d-d_\text{est}\right)},
	\end{align}
	and where $d_{est} := \sum_ma_m^*g(x_m^*,\Sigma_m^*)$. In all minimization sub-procedures we used the L-BFGS-B method with non-negativity constraints on the weights and ran it until convergence. Hyper-parameters $\tau_1$ and $\tau_2$ (introduced in line 5 and 7 in Algorithm 1) should ideally be chosen based on the amount of noise in the data, however we found that the exact values were not so important in our toy examples, and we used the ad-hoc chosen values $\tau_1=10$ and $\tau_2=20$. We leave a more careful sensitivity analysis with respect to these parameter to future studies. The total run-time was a few minutes on a computer with an Intel Pentium CPU running at 2.90GHz, using $\sim 8$ GB of RAM. The final results of our experiments are tabulated in Table \ref{tab:Exp1} \--- \ref{tab:Exp3} and are plotted in Figure \ref{fig:1d_exa}\---\ref{fig:2d_exa_8_circs}. Intermediate results for the 2D experiments are plotted in Figure \ref{fig:2d_exa_elongated_iters} \--- \ref{fig:2d_exa_8_circs_iters}.

	\begin{figure}
		\centering
		\includegraphics[width=1.0\linewidth]{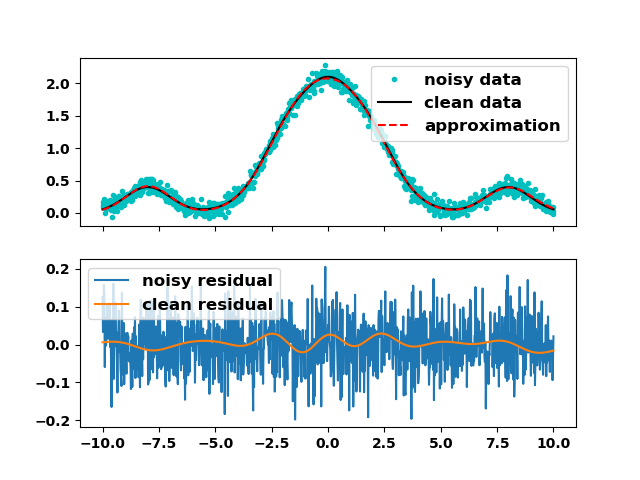}
		\\ 
		\caption{Results from Experiment 1.}\label{fig:1d_exa}
	\end{figure}

	\begin{figure}
		\centering
		\includegraphics[width=0.45\linewidth]{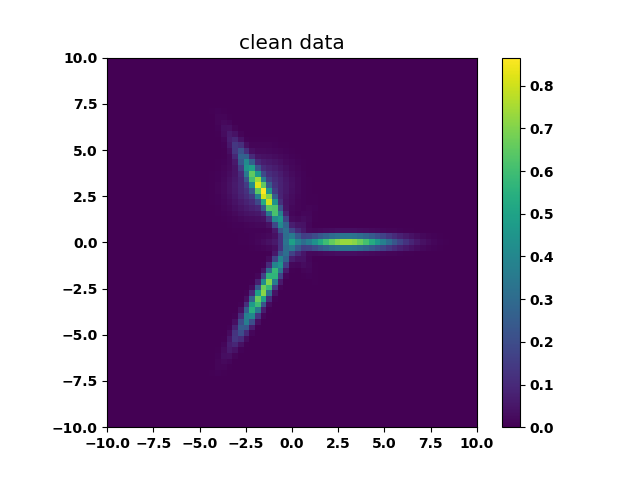}	 
	    \includegraphics[width=0.45\linewidth]{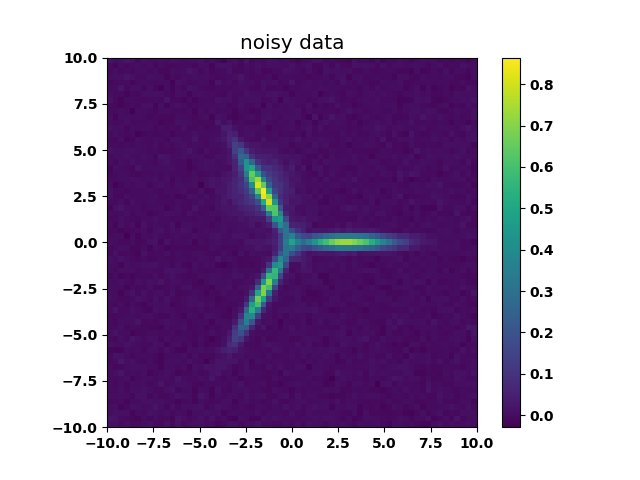}
		\includegraphics[width=0.45\linewidth]{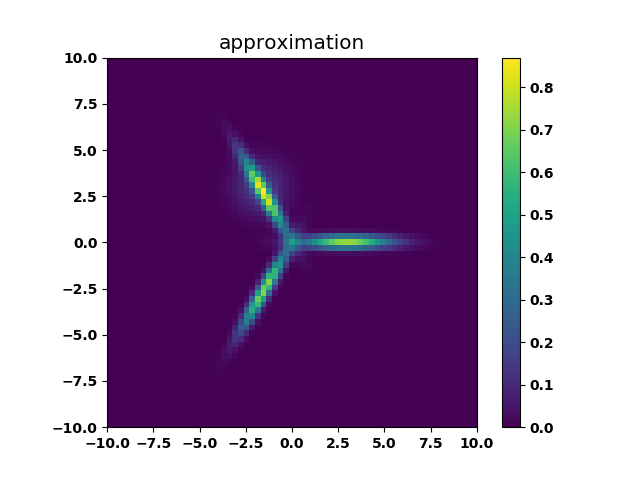}
		\includegraphics[width=0.45\linewidth]{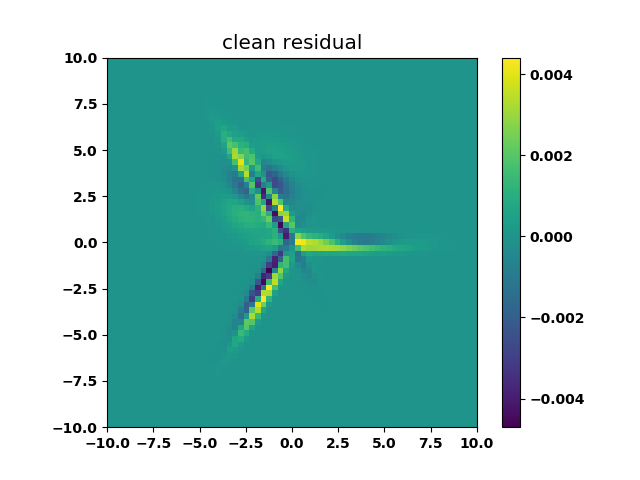}
		\includegraphics[width=0.45\linewidth]{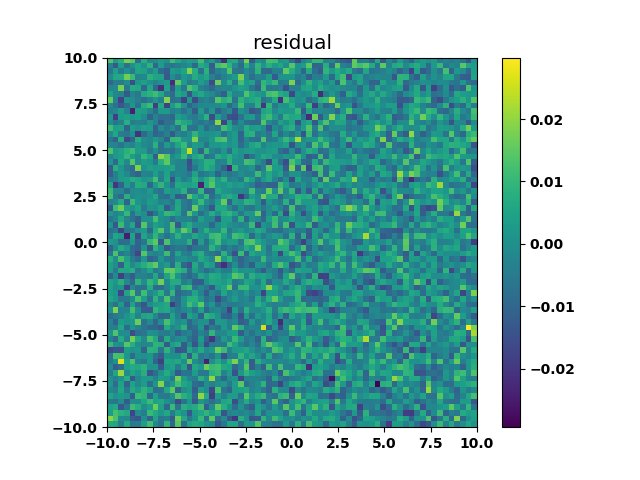}
		\caption{Results from Experiment 2. Note the faint spherical Gaussian on the upper left "leg".}\label{fig:2d_exa_elongated}
	\end{figure}
	\begin{figure}
		\centering
		\includegraphics[width=0.45\linewidth]{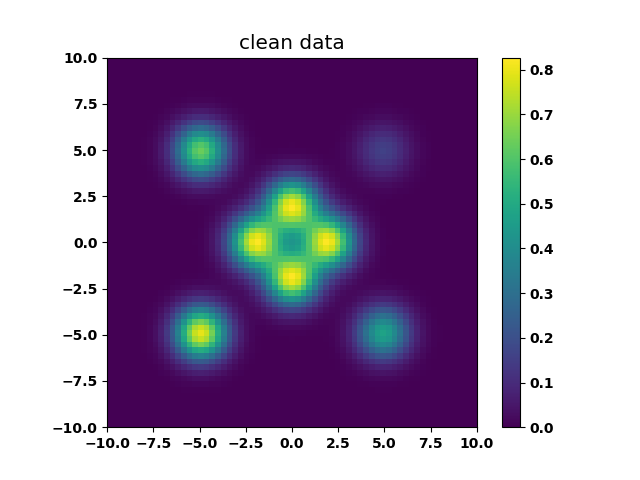}	 
		\includegraphics[width=0.45\linewidth]{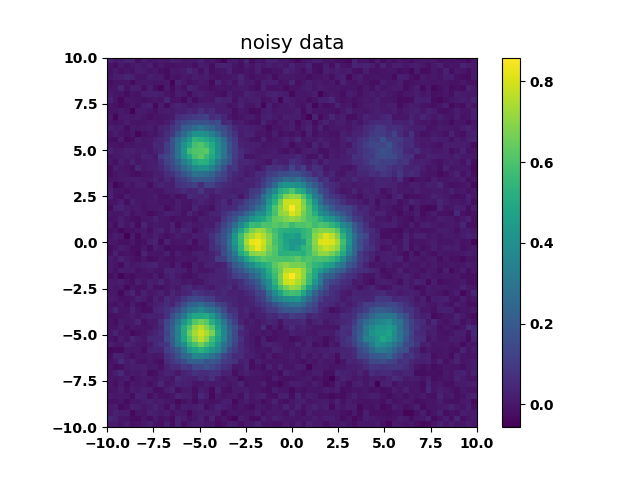}
		\includegraphics[width=0.45\linewidth]{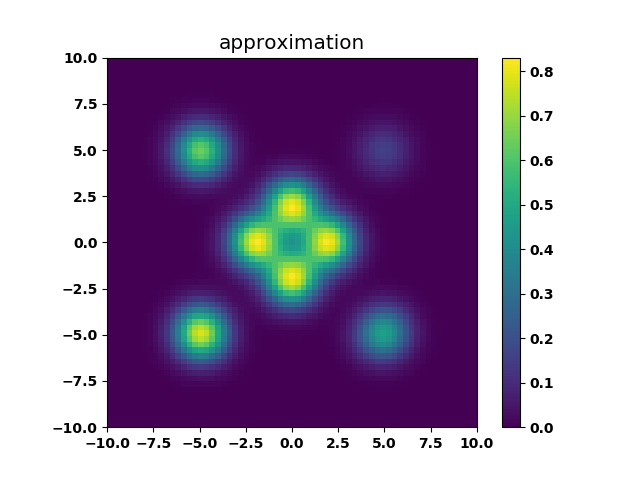}
		\includegraphics[width=0.45\linewidth]{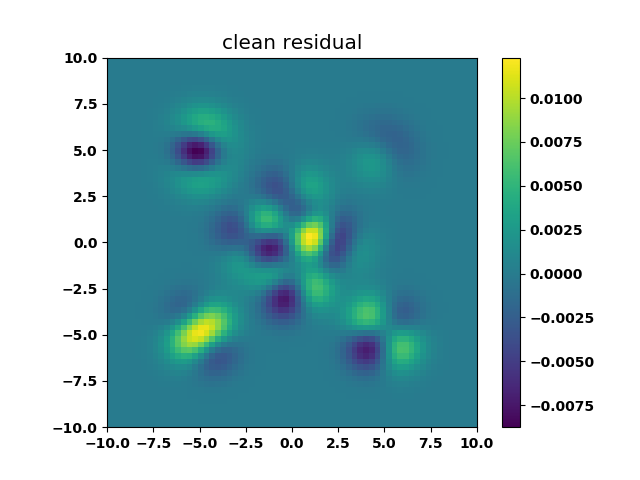}
		\includegraphics[width=0.45\linewidth]{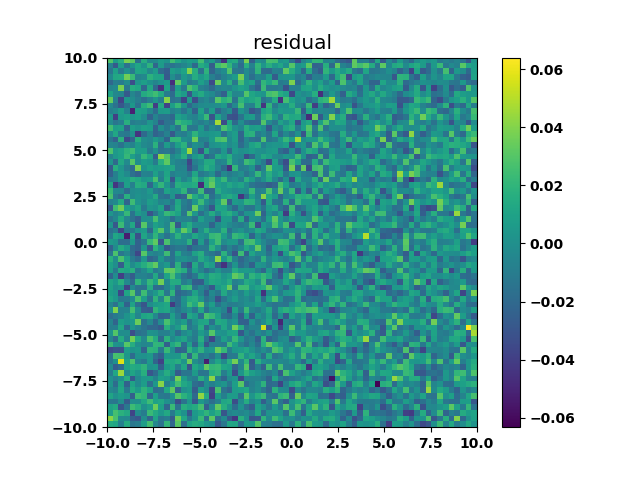}
		\caption{Results from Experiment 3.}\label{fig:2d_exa_8_circs}
	\end{figure}
	
	\begin{figure}
		\centering
		\includegraphics[width=0.45\linewidth]{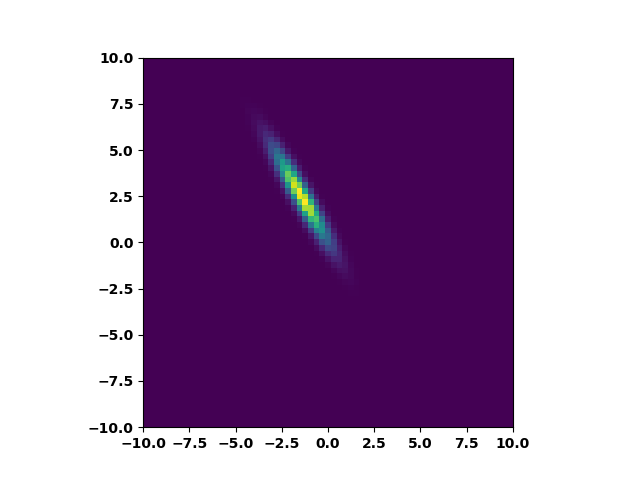}
		\includegraphics[width=0.45\linewidth]{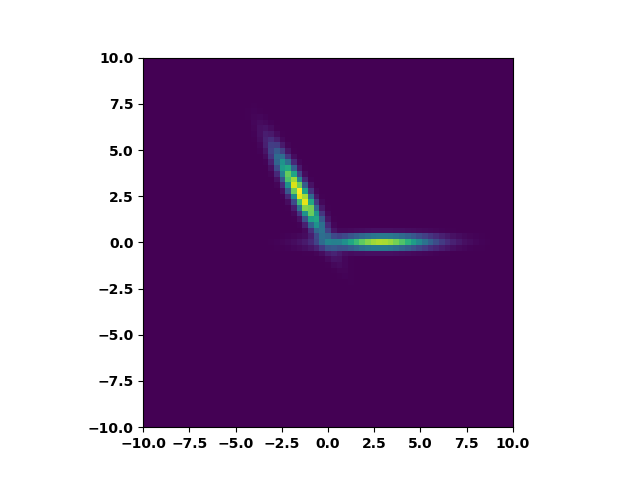}
		\includegraphics[width=0.45\linewidth]{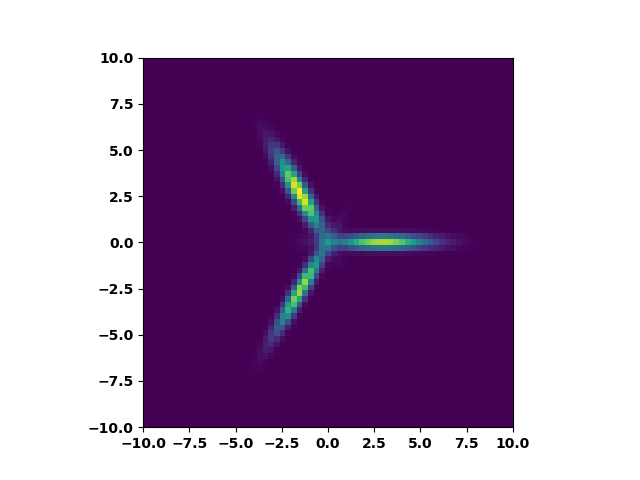}
		\includegraphics[width=0.45\linewidth]{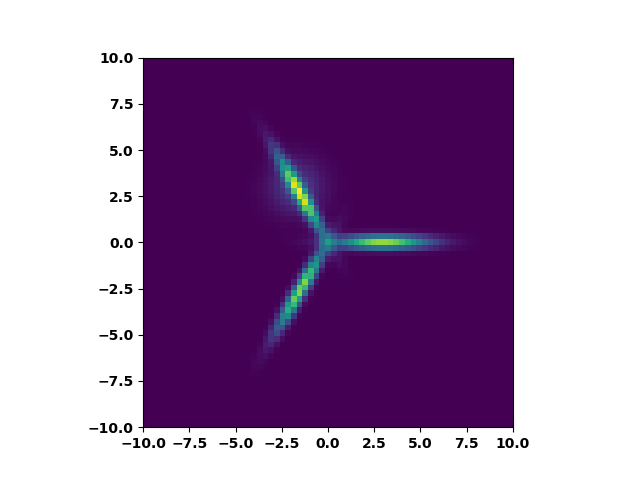}
		\caption{Iterations from Experiment 2.}\label{fig:2d_exa_elongated_iters}
	\end{figure}
	
	\begin{figure}
		\centering
		\includegraphics[width=0.45\linewidth]{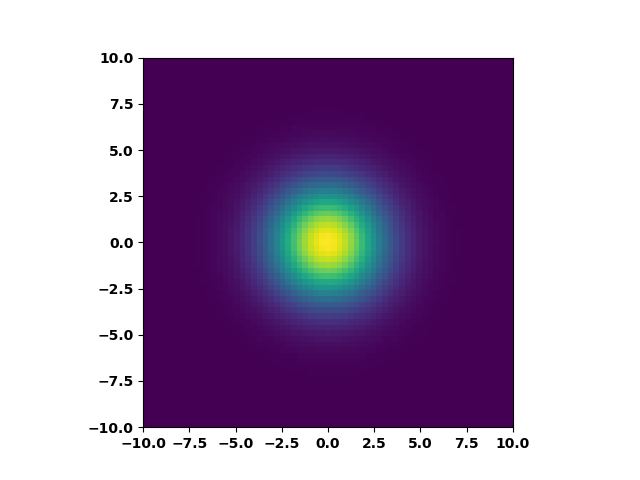}
		\includegraphics[width=0.45\linewidth]{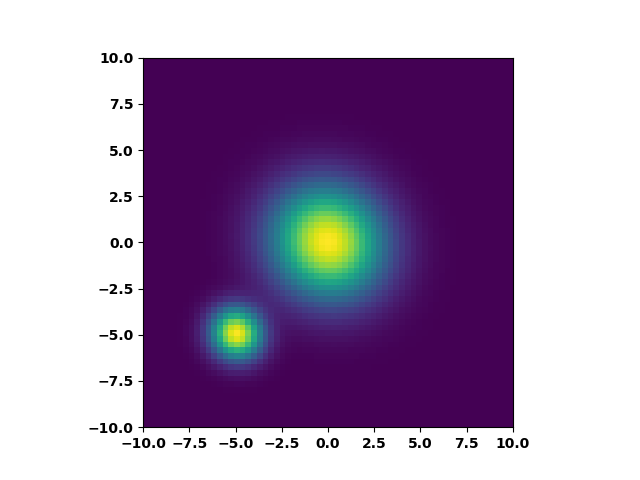}
		\includegraphics[width=0.45\linewidth]{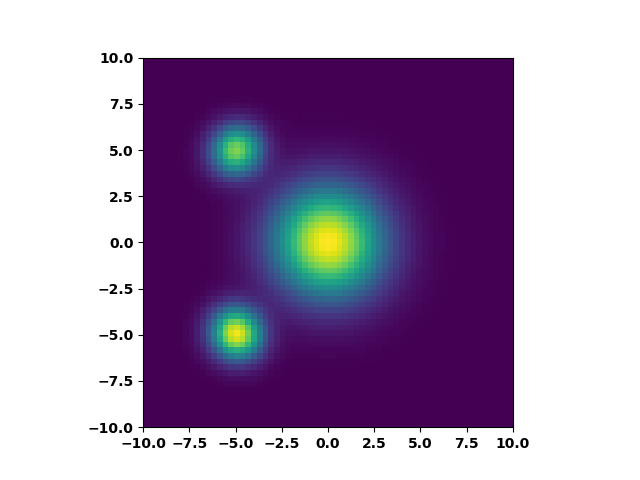}
		\includegraphics[width=0.45\linewidth]{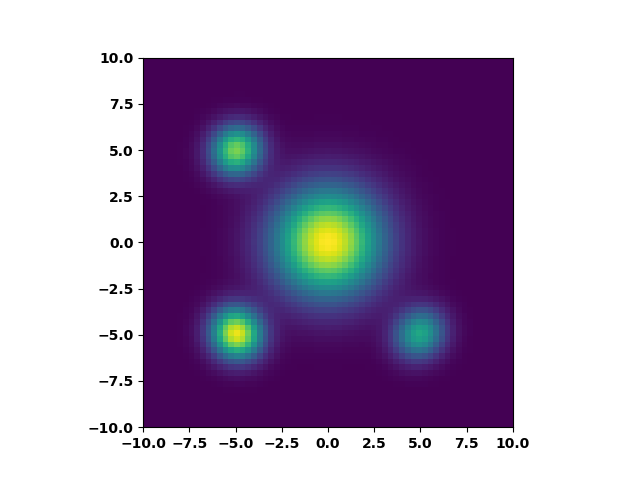}
		\includegraphics[width=0.45\linewidth]{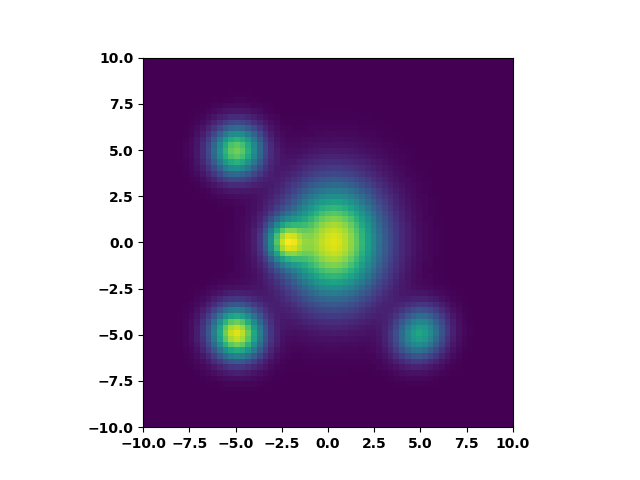}
		\includegraphics[width=0.45\linewidth]{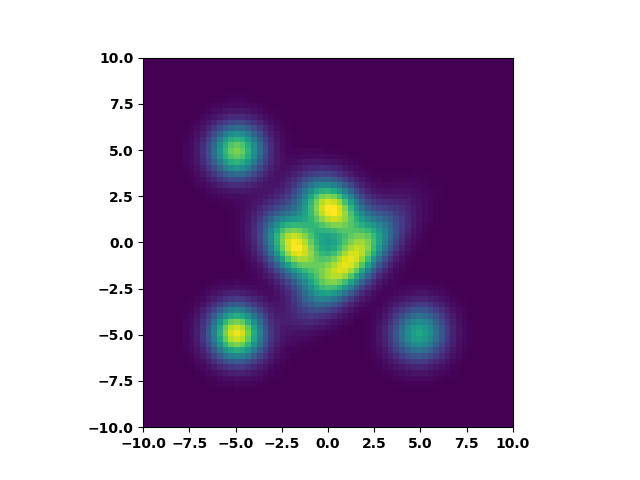}
		\includegraphics[width=0.45\linewidth]{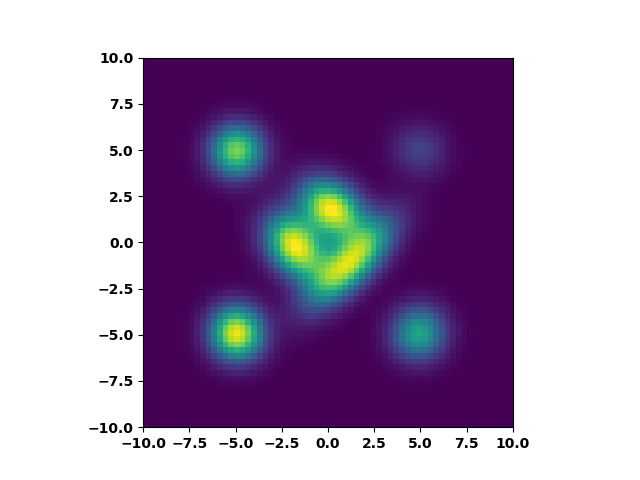}
		\includegraphics[width=0.45\linewidth]{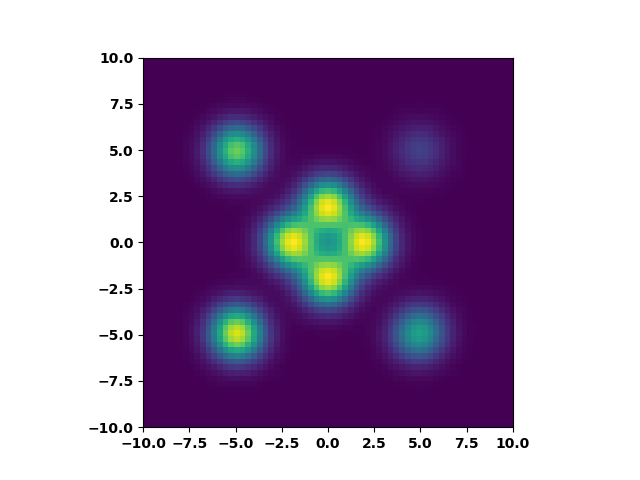}
		\caption{Iterations from Experiment 3.}\label{fig:2d_exa_8_circs_iters}
	\end{figure}

	\begin{remark}
		The number of modes of a GMM is not necessarily equal to the number of means. Indeed, there might be more modes than means \cite{amendola2017maximum}. Alternatively, there could be fewer modes than means, for example, in Experiment 2, a low amplitude Gaussian and high amplitude Gaussian are close to each other, and there is only one mode within the vicinity of both means.  Our method successfully recovered the parameters of both Gaussians. There is also a mode at the center dominantly formed by the superposition of the tails of three anisotropic Gaussians, where our algorithm did not introduce a spurious Gaussian. Similar performance is observed in Experiment 3. We attribute this to the greedy nature of the algorithm where a predefined number of Gaussians is not enforced in the decomposition.
	\end{remark}
	
	\subsection{Clustering and comparison to expectation maximization}
As one of the main techniques in unsupervised learning, clustering seeks to subdivide data into groups based on some similarity measure. A common approach in clustering is based on Gaussian mixtures. In this approach one is given a point-cloud $P\subset\R^n$ and the goal is to find a GMM $f$ such that the probability that $P$ is a sample from $f$ is maximal among all GMMs with a prescribed number of components. 

There are many ways of transforming a point-cloud into a signal and vice-versa. Thus one may use the method proposed in this paper for GMM-based clustering, and conversely one may use methods from clustering in order to decompose a signal into a GMM. In the below we illustrate both directions of this "problem transformation" with some numerical examples. 

\subsubsection{Clustering via signal decomposition}
A sample point-cloud $P$ of size $10^5$ was drawn from a normalized version of the GMM in Experiment 3 above. Then a signal $d$ was generated from $P$ by computing a histogram based on the 2D-grid used in the previous experiments. Results from the expectation maximization (EM) algorithm and the proposed method applied to $P$ and $d$, respectively, are shown in Figure \ref{fig:clust_via_tensor_decomp}. The two methods performed similarly in terms of the likelihood of $P$ given the computed GMMs.
 %	Thus $d(y) = \lbrace\# \text{points from P in \rbrace$
 	\begin{figure}
 		\centering
 		\includegraphics[width=0.70\linewidth]{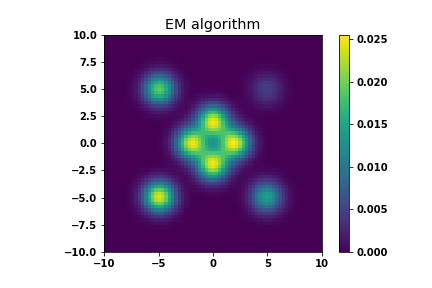}	
 		\includegraphics[width=0.70\linewidth]{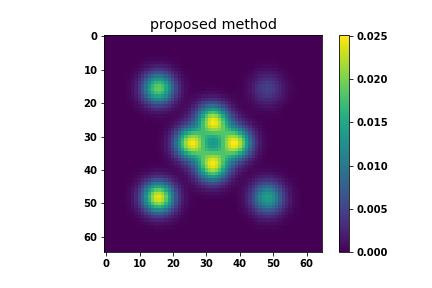}
 		\caption{Results from clustering via signal decomposition.}\label{fig:clust_via_tensor_decomp}
 	\end{figure}
\subsubsection{Signal decomposition via clustering}
A noisy signal $d$ was generated as in Experiment 3 described in section \ref{sec:mainNumerical}, except that the GMM was normalized before discretization. To $d$ we then associated a point-cloud $P$ following the methodology in \cite{Jou16}, so $P$ had $\left \lfloor{Cd(y)}\right \rfloor$ points at grid-point $y$, where the constant $C := 10^{-5}\sum_yd(y)$ was chosen so that $P$ had roughly $10^5$ points in total. Using $P$ we estimated a GMM with the EM algorithm. The performance of EM depended on the random initialization. A typical result using EM applied to $P$, along with the result from the proposed method applied to $d$ are shown in Figure \ref{fig:tensor_decomp_via_clust}. We remark that is possible that a better result with EM may be obtained using some other procedure for generating the point-cloud from the given signal.

	\begin{figure}
	\centering
	\includegraphics[width=0.70\linewidth]{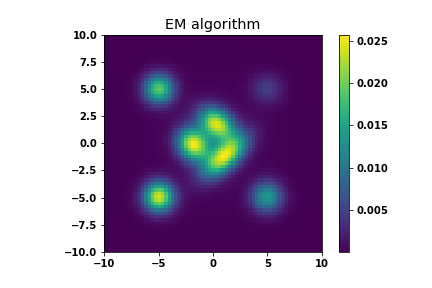}	
	\includegraphics[width=0.70\linewidth]{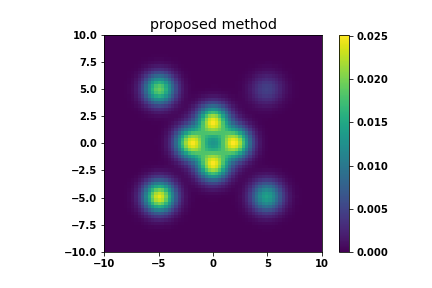}
	\caption{Results from signal decomposition via clustering.}\label{fig:tensor_decomp_via_clust}
\end{figure}

	\section{Theoretical considerations}\label{sec:theory}
		The main theoretical contribution of our paper is Theorem 3.4, which says that a mode of a GMM can not lie too far away from the set of means. This theorem generalizes the 1D result that any mode will lie within one standard deviation away from some mean, and provides theoretical support for our way of initializing each iteration of our method. Before turning to this theorem, we need three lemmas. In the first two lemmas we compute some integrals over the $n$-sphere of certain polynomials, and in the third lemma we provide expressions for the gradient and Hessian of a GMM. 	We also prove, in Proposition \ref{prop:GaussApproximation}, that any sufficiently nice non-negative function may be approximated in $L^\infty$ to any desired accuracy by a GMM.
	\subsection{Location of modes of Gaussian mixtures}		

	\begin{lemma}\label{lem:IntInner}
		\begin{align}
		\int_{S^{n-1}}(x, y)^k dy = |x|^kC_k,\quad x\in\R^n, k\in\mathbb{N}, 
		\end{align}
		where $C_k$ is a constant only depending on $k$.
	\end{lemma}
	\begin{proof}
		WLOG assume $x\neq 0$. 
		\begin{align}
		\int_{S^{n-1}}(x, y)^k dy &= |x|^k\int_{S^{n-1}}\left(\frac{x}{|x|}, y\right)^k dy = |x|^kC_k,
		\end{align}
		where the constant $C_k$ is defined by  $C_k := \int_{S^{n-1}}\left(e_1, y\right)^k dy$.
	\end{proof}
	\begin{lemma}\label{lem:IntQuadratic}
		Let $A$ be a symmetric $n\times n$ matrix. Then
		\begin{align}
		\int_{S^{n-1}}(y, Ay)dy = \text{Tr}(A)C_2,
		\end{align}
		where $C_2$ is as in Lemma \ref{lem:IntInner}.
	\end{lemma}
	\begin{proof}
		Let $\lambda_1,\dots, \lambda_n$ be the eigenvalues of $A$ and take $U\in\text{O}(n)$ such that $A=U^TDU$, where $D:=\text{diag}\left(\lambda_1,\dots,\lambda_n\right)$.
		\begin{align}
		&\int_{S^{n-1}}(y, Ay)dy = \int_{S^{n-1}}(y, U^TDUy)dy \\
		 &= \int_{S^{n-1}}(Uy, DUy)dy = \int_{S^{n-1}}(y, Dy)dy \\
		 &= \sum_{m=1}^n\lambda_m\int_{S^{n-1}}(y, e_me_m^Ty)dy \\
		&= \sum_{m=1}^n\lambda_m\int_{S^{n-1}}(e_m, y)^2dy = C_2\sum_{m=1}^n\lambda_m = C_2\text{Tr}(A).
		\end{align}
	\end{proof}
	\begin{lemma}\label{lem:GMDeriv}
		The gradient and Hessian of a Gaussian mixture model \footnote{In the interest of readability we have abused notation and absorbed the normalizing constants $C_{\Sigma_m}$ into the weights $a_m$.}
		\begin{align}
		f(x) = \sum_{m}a_m\exp\left\lbrace-\frac{1}{2}\left(x-x_m\right)^T\Sigma_m^{-2}\left(x-x_m\right)\right\rbrace,
		\end{align}
		$a_m>0, \Sigma_m>0,$ are given by
		\begin{align}
		\nabla f &= \sum_{m}a_m\left(-\Sigma_m^{-1}\right)\left(\Sigma_m^{-1}x-\Sigma_m^{-1}x_m\right)\\
		&\qquad\qquad\times \exp\left\lbrace-\frac{1}{2}\left(x-x_m\right)^T\Sigma_m^{-2}\left(x-x_m\right)\right\rbrace \\
		Hf &= \sum_ma_m\left[\left(\Sigma_m^{-2}\left(x-x_m\right)\right)\left(\Sigma_m^{-2}\left(x-x_m\right)\right)^T - \Sigma_m^{-2}\right]\\
		&\qquad\qquad\times\exp\left\lbrace-\frac{1}{2}\left(x-x_m\right)^T\Sigma_m^{-2}\left(x-x_m\right)\right\rbrace
		\end{align}
		
	\end{lemma}
	\begin{proof}
		Since this is a standard result, we omit the straightforward proof.
	\end{proof}
	
	\begin{theorem}\label{thm:modeDist}
		Consider a Gaussian mixture model in $n$ dimensions
		\begin{align}
		f(x) = \sum_{m}a_m\exp\left\lbrace-\frac{1}{2}\left(x-x_m\right)^T\Sigma_m^{-2}\left(x-x_m\right)\right\rbrace,
		\end{align}
		$a_m>0, \Sigma_m>0.$ Let $\sigma_{m,\max}$ and $\sigma_{m,\min}$ be the maximal and minimal eigenvalues of $\Sigma_m$. Let $x'$ be a local maximum of $f$. Then there exists an index $m$ such that
		\begin{align}
		|x'-x_m| \leq \sqrt{n}\sigma_{m,\max}^2\sigma_{m,\min}^{-1}.
		\end{align} 
		
	\end{theorem}
	\begin{proof} 
		Since $x'$ is a local maximum we have $Hf(x') \leq 0$, i.e. $\left(y, Hf(x')y\right)\leq 0, y\in\R^n$. We integrate the last inequality over the unit-sphere and make use of Lemma \ref{lem:GMDeriv} to conclude:
		\begin{align}
		\sum_ma_m\int_{S^{n-1}}y^TA_m(x',y)ydy\exp\left\lbrace\cdots\right\rbrace  \leq 0,
		\end{align}
		where\\
		 $A_m(x',y):= \left[\left(\Sigma_m^{-2}\left(x'-x_m\right)\right)\left(\Sigma_m^{-2}\left(x'-x_m\right)\right)^T - \Sigma_m^{-2}\right]$.
		Hence there exist some index $m$ such that 
		\begin{align}
		\int_{S^{n-1}}y^TA_m(x',y)ydy & \leq 0,
		\end{align}
		which leads to
		\begin{align}
		\int_{S^{n-1}}\left(y,\Sigma_m^{-2}\left(x'-x_m\right)\right)^2dy -\int_{S^{n-1}}\left(y,\Sigma_m^{-2}y\right)dy & \leq 0.
		\end{align}
		At this point we apply Lemma \ref{lem:IntInner} and Lemma \ref{lem:IntQuadratic} and obtain:
		\begin{align}
		C_2\left|\Sigma_m^{-2}\left(x'-x_m\right)\right|^2 -C_2\text{Tr}\left(\Sigma_m^{-2}\right) & \leq 0.
		\end{align}
		Now $C_2> 0$ since $C_2$ is an integral of a continuous non-negative function that is not everywhere zero. Hence
		\begin{align}
		\left|\Sigma_m^{-2}\left(x'-x_m\right)\right|^2 \leq \text{Tr}\left(\Sigma_m^{-2}\right).
		\end{align}
		Note that 
		\begin{align}
		\sigma_{m,\max}^{-4}\left|x'-x_m\right|^2 \leq \left|\Sigma_m^{-2}\left(x'-x_m\right)\right|^2
		\end{align}
		and that
		\begin{align}
		\text{Tr}\left(\Sigma_m^{-2}\right) \leq n\sigma_{m,\min}^{-2}.
		\end{align}
		So
		\begin{align}
		\sigma_{m,\max}^{-4}\left|x'-x_m\right|^2 \leq n\sigma_{m,\min}^{-2},
		\end{align} 
		and the claim of the theorem follows.
	\end{proof}
	
	\begin{cor}\label{cor:sphericalBound}
		Let $f$ be a mixture of spherical Gaussians with common variance $\sigma^2$, i.e.
		\begin{align}
		f(x) = \sum_{m}a_m\exp\left\lbrace-\frac{1}{2\sigma^2}\left|x-x_m\right|^2\right\rbrace, \sigma >0, a_m>0. 
		\end{align}
		If $x'$ is local maximum of $f$, then there is an index $m$ such that 
		\begin{align}
		|x'-x_m| \leq \sqrt{n}\sigma.
		\end{align} 
	\end{cor}
	\begin{proof}
		This is an immediate consequence of Theorem \ref{thm:modeDist}.
	\end{proof}
	\begin{prop}
		The bound in Theorem \ref{thm:modeDist} cannot be improved by a constant, i.e. for any $\delta>0$ there exist a GMM $f$ such that some mode $x'$ of $f$ satisfies 
		\begin{align}
		|x'-x_m| > \sqrt{n}\sigma_{m,\max}^2\sigma_{m,\min}^{-1} - \delta, 
		\end{align} 
		for all mean vectors $x_m$ of $f$.
	\end{prop}
	\begin{proof}
		We explicitly construct a family $\{f_\epsilon\}_\epsilon$ of functions that satisfy the statement of this proposition. For $\epsilon$ such that $\min\left(\delta,\sqrt{n}\sigma\right) > \epsilon>0$ let $f_\epsilon$ be the $2n$ component $n$-dimensional spherical GMM with common variance $\sigma^2$, common amplitude $a$ and with means at $\pm (\sqrt{n}\sigma - \epsilon) e_{i}$, for $i=1,2,\dots,n$.  We shall prove that $f_\epsilon$ has a mode in the origin, and thus it has a mode at distance $\sqrt{n}\sigma - \epsilon$ from the set of means of $f_\epsilon$.\footnote{Numerical experiments suggest that $f_\epsilon$ has a mode in the origin also for $\epsilon=0$. A proof of this (if it is true) would however need an argument different from the one given here, since $Hf_0(0) = 0$.} Lemma \ref{lem:GMDeriv} implies
		\begin{align}
		\nabla f_\epsilon(0) &= a\sum_{m}\sigma^{-2}x_m\exp\left\lbrace-\frac{1}{2\sigma^2}\lvert x_m\rvert^2\right\rbrace \\ 
		&= a\sigma^{-2}\exp\left\lbrace-\frac{\left(\sqrt{n}\sigma-\epsilon\right)^2}{2\sigma^2}\right\rbrace\sum_mx_m. 
		\end{align}
		By symmetry $\sum_{m}x_m=0$, so $f_\epsilon$ has a critical point in the origin. Hence we are done if we prove that $Hf_\epsilon(0) < 0$. Again by Lemma \ref{lem:GMDeriv}:
		\begin{align}
		Hf_\epsilon(0) &= a\sum_m\left(\sigma^{-4}x_mx_m^T - \sigma^{-2}I_n\right)e^{-\frac{1}{2\sigma^2}\lvert x_m\rvert^2} \\
		&= 2ae^{-\frac{\left(\sqrt{n}\sigma-\epsilon\right)^2}{2\sigma^2}}\sum_{i=1}^n\left(\sigma^{-4}(\sqrt{n}\sigma - \epsilon)^2 e_{i} e_{i}^T - \sigma^{-2}I_n\right) \\
		&= \underbrace{2ae^{-\frac{\left(\sqrt{n}\sigma-\epsilon\right)^2}{2\sigma^2}}\sigma^{-2}\left[\sigma^{-2}(\sqrt{n}\sigma - \epsilon)^2- n\right]}_{<0}I_n.
		\end{align} 
		Hence $Hf_\epsilon(0) < 0$.
	\end{proof}
\subsection{Function approximation by Gaussian mixtures}
The ability of GMMs to approximate functions in $L^p$ spaces has been investigated previously, see e.g. \cite{nestoridis2011universal} where it is noted that any probability density function may be approximated in the sense of $L^1$ by GMMs. For completeness, we here give a proof of the density of GMMs in the $L^\infty$-norm. The proof relies on the machinery of \emph{quasi interpolants} \cite{mazia2007approximate}.
\begin{prop}\label{prop:GaussApproximation}
Let $u$ be a twice differentiable and compactly supported non-negative function on $\R^n$ such that $u$ and all its partial derivatives up to order two are bounded. Then $u$ may be approximated in $L^\infty$ arbitrarily well by Gaussian mixtures.
\end{prop}
\begin{proof}
 Let $g = g(0, \Sigma)$ be an $n$-dimensional centred Gaussian density function with non-degenerate covariance matrix $\Sigma^2$. As is noted on page 50 in \cite{mazia2007approximate}, $g$ generates an approximate quasi interpolant of order 2. This follows from the observation that $f$ satisfies both Condition 2.15 in \cite[p. 33]{mazia2007approximate}, with moment order $N=2$, and Condition 2.12 in \cite[p. 32]{mazia2007approximate}, with decay order $K$, for all $\mathbb{Z}\ni K > n$. Hence Theorem 2.17 in \cite[p. 35]{mazia2007approximate} is applicable and implies that for any twice differentiable real-valued function $u$ on $\R^n$ such that $u$ and all its partial derivatives up to order two are bounded, one has that for any $\epsilon>0$ that there exists a $\mathcal{D} = \mathcal{D}(\epsilon)>0$ such that for all~h, $\left\lvert u - \mathcal{M}_{h,\mathcal{D}}u\right\lvert_\infty$ is bounded from the above by
\begin{align}
  &c \mathcal{D}h^2\max_{|\alpha|=2}\lvert \partial^\alpha u\rvert_\infty + \epsilon \left(\lvert u\rvert_\infty + \sqrt{\mathcal{D}}h\lvert\nabla u\rvert_\infty\right) \\
&\leq A \left(c\mathcal{D}h^2+ \epsilon + \sqrt{\mathcal{D}}h\right),
\end{align}
where $A:= \sqrt{n}\max_{0 \leq |\alpha| \leq 2} \lvert\partial^\alpha u\rvert_\infty$, $c$ is a constant independent of $u, h$ and $\mathcal{D}$ and the quasi interpolant $\mathcal{M}_{h,\mathcal{D}}u(x)$ is defined by
\begin{align}\label{eq:QuasiDef}
\mathcal{M}_{h,\mathcal{D}}u(x) := \mathcal{D}^{-n/2}\sum_{m\in\mathbb{Z}^n}u(hm)g\left(\frac{x-hm}{\sqrt{\mathcal{D}}h}\right).
\end{align}
Note that the weights $u(hm)$ are non-negative since $u$ is non-negative, and that the compact support of $u$ allows us to restrict the domain of summation in \eqref{eq:QuasiDef} to a finite subset of $\mathbb{Z}^n$. Hence $\mathcal{M}_{h,\mathcal{D}}u(x)$ is a GMM. Now for a given $\epsilon>0$ we let $\epsilon' := \frac{\epsilon}{2A}$ and pick $\mathcal{D}>0$ such that:
\begin{align}
\left\lvert u - \mathcal{M}_{h,\mathcal{D}}u\right\lvert_\infty 
& \leq A \left(c\mathcal{D}h^2+ \epsilon' + \sqrt{\mathcal{D}}h\right)\\ 
&=  \epsilon/2 + A \left(c\mathcal{D}h^2 + \sqrt{\mathcal{D}}h\right).
\end{align}
Next we take $h$ small enough, so that $A \left(c\mathcal{D}h^2 + \sqrt{\mathcal{D}}h\right) \leq \epsilon/2$, and conclude that
\begin{align}
\left\lvert u - \mathcal{M}_{h,\mathcal{D}}u\right\lvert_\infty 
&\leq \epsilon.
\end{align}
\end{proof}
\section{Conclusion} \label{sec:conc}
Motivated primarily by applications in TEM, we have developed a new algorithm for decomposing a non-negative multivariate signal as a sum of Gaussians with full covariances. We have tested it on 1D and 2D data. Moreover, we have also proved an upper bound for the distance from a local maximum of a GMM to the set of its mean vectors. This upper bound provides motivation for a key step in our method, namely the  initialization of each new Gaussian at the maximum of the residual.  Finally we remark that, while we have only tested the proposed method on functions sampled on uniform grids, it is straightforward to extend the method to handle input data in the form of multivariate functions sampled on \emph{non-uniform} grids.

	\newpage
	\bibliographystyle{alpha}
	\bibliography{gmm_paper_preprint}

\begin{thebibliography}{KBGP18}

\bibitem[AEH17]{amendola2017maximum}
Carlos Am{\'e}ndola, Alexander Engstr{\"o}m, and Christian Haase.
\newblock Maximum number of modes of gaussian mixtures.
\newblock {\em arXiv preprint arXiv:1702.05066}, 2017.

\bibitem[BLNZ95]{byrd1995limited}
Richard~H Byrd, Peihuang Lu, Jorge Nocedal, and Ciyou Zhu.
\newblock A limited memory algorithm for bound constrained optimization.
\newblock {\em SIAM Journal on Scientific Computing}, 16(5):1190--1208, 1995.

\bibitem[CP00]{CP00}
Miguel~{\' A}. Carreira-Perpi{\~ n}{\' a}n.
\newblock Mode-finding for mixtures of {G}aussian distributions.
\newblock Technical report, Dept. of Computer Science, University of Sheffield,
  UK, 2000.

\bibitem[F{\c C}99]{FC99}
Khaled~Ben Fatma and A.~Enis {\c C}.
\newblock Design of gaussian mixture models using matching pursuit.
\newblock {\em IEEE-EURASIP Workshop on Nonlinear Signal and Image Processing},
  1999.

\bibitem[Jou16]{Jou16}
Paul Joubert.
\newblock {\em A Bayesian approach to initial model inference in cryo-electron
  microscopy}.
\newblock PhD thesis, Georg-August University School of Science, 2016.

\bibitem[JS16a]{JS16a}
Slavica Joni{\' c} and Carlos {\' O}scar~S{\' a}nchez Sorzano.
\newblock Coarse-graining of volumes for modeling of structure and dynamics in
  electron microscopy: Algorithm to automatically control accuracy of
  approximation.
\newblock {\em IEEE Journal of Selected Topics in Signal Processing},
  10(1):161--173, 2016.

\bibitem[JS16b]{JS16b}
Slavica Joni{\' c} and Carlos {\' O}scar~S{\' a}nchez Sorzano.
\newblock Versatility of approximating single-particle electron density maps
  using pseudoatoms and approximation-accuracy control.
\newblock {\em BioMed Research International}, 2016.

\bibitem[Kaw18]{Kaw18}
Takeshi Kawabata.
\newblock Gaussian-input {G}aussian mixture model for representing density maps
  and atomic models.
\newblock {\em Journal of Structural Biology}, 203:1--16, 2018.

\bibitem[KBGP18]{KBGP18}
Nicolas Keriven, Anthony Bourrier, R{\' e}mi Gribonval, and Patrick P{\' e}rez.
\newblock Sketching for large-scale learning of mixture models.
\newblock {\em Information and Inference}, 7(3):447--508, 2018.

\bibitem[MMS07]{mazia2007approximate}
Vladimir~Gilelevich Mazia, Vladimir~Gilelevic Maza, and Gunther Schmidt.
\newblock {\em Approximate approximations}.
\newblock Number 141. American Mathematical Soc., 2007.

\bibitem[MZ93]{mallat1993matching}
St{\'e}phane~G Mallat and Zhifeng Zhang.
\newblock Matching pursuits with time-frequency dictionaries.
\newblock {\em IEEE Transactions on signal processing}, 41(12):3397--3415,
  1993.

\bibitem[NSS11]{nestoridis2011universal}
Vassili Nestoridis, Sebastian Schmutzhard, and Vangelis Stefanopoulos.
\newblock Universal series induced by approximate identities and some relevant
  applications.
\newblock {\em Journal of approximation theory}, 163(12):1783--1797, 2011.

\bibitem[PRK93]{pati1993orthogonal}
Yagyensh~Chandra Pati, Ramin Rezaiifar, and Perinkulam~Sambamurthy
  Krishnaprasad.
\newblock Orthogonal matching pursuit: Recursive function approximation with
  applications to wavelet decomposition.
\newblock In {\em Proceedings of 27th Asilomar conference on signals, systems
  and computers}, pages 40--44. IEEE, 1993.

\end{thebibliography}

\end{document}